\documentclass{article}

\PassOptionsToPackage{dvipsnames}{xcolor}  % avoid option clash: tikz loads xcolor first
\usepackage[final]{neurips_2026}

\makeatletter
\renewcommand{\@noticestring}{%
  }
\makeatother

\usepackage[utf8]{inputenc}
\usepackage[T1]{fontenc}
\usepackage{microtype}
\usepackage{amsmath, amssymb, amsthm}
\usepackage{booktabs}
\usepackage{graphicx}
\usepackage{algorithmic}
\usepackage{enumitem}
\usepackage{bm}
\usepackage{tikz}
\usetikzlibrary{arrows.meta, positioning, calc}
\usepackage{multirow}
\usepackage{subcaption}
\usepackage{float}
\usepackage{xcolor}
\usepackage{hyperref}
\usepackage{cleveref}
\usepackage{url}
\usepackage{makecell}
\usepackage{booktabs}
\usepackage{amsmath, amssymb, amsthm}

\newcommand{\iqr}[1]{{\scriptsize [#1]}}
\newtheorem{theorem}{Theorem}

\title{Emergent Latent-State Computation under Stochastic Volatility}

\author{
Xiaoyu Huang$^{1}$ \\
Department of Mathematics\\
Temple University\\
Philadelphia, PA, USA
\And
Lulu Wang$^{2,}$\thanks{Corresponding author: \texttt{wanglu@dickinson.edu}. Code available at \url{https://github.com/LuluWangGitHubt/SVTransformer}.} \\
Department of Data Analytics\\
Dickinson College\\
Carlisle, PA, USA
}

\begin{document}

\maketitle

\begin{abstract}

Mechanistic interpretability has largely focused on language models and deterministic toy tasks.
Much less is known about how sequence models internally represent latent stochastic dynamics under noisy, partially observed observations.
We study this question in a controlled multivariate stochastic volatility setting, where models observe only returns while the ground-truth latent volatility state is known to the researcher.
This setting provides a useful benchmark for mechanistic interpretability under partial observability: the latent state is hidden from the model but directly available for evaluation.
Across architectures, losses, and output heads, we find evidence for a two-stage computation.
Hidden representations encode substantial information about the next latent volatility state, and the output head maps this representation to squared return forecasts.
Furthermore, in Transformers, latent-state decodability emerges at identifiable architectural stages whose location depends on the volatility period. In long-cycle regimes, this computation simplifies into an explicit latent-state filter consisting of a learned linear projection followed by $\ell^2$ normalization. Output-head replacement further shows that part of the degradation under noisy MSE training arises from readout misalignment rather than representation failure. These results suggest that stochastic volatility models provide a useful benchmark for mechanistic interpretability under noisy latent dynamics and partial observability.

\end{abstract}

\section{Introduction}
\label{sec:intro}
Mechanistic interpretability has largely focused on language models and synthetic algorithmic tasks \citep{wang2022interpretability, bricken2023monosemanticity, nanda2023progress, chughtai2023toy}, where the target computation is often deterministic and fully observable. Much less is known about how neural sequence models represent latent stochastic dynamics in noisy, partially observed settings.  Stochastic volatility models \citep{taylor1986modelling,kim1998stochastic,shephard1996statistical} provide a natural testbed: a central class of financial econometric models in which volatility evolves as an unobserved stochastic state rather than as a deterministic function of past returns, as in ARCH/GARCH models \citep{engle1982autoregressive,bollerslev1986generalized,shephard1996statistical}. In this setting, returns are noisy observations of the latent process, and accurate forecasting often requires an implicit filtering computation that infers hidden states from past observations.

We study whether small Transformers trained only to forecast future squared returns recover internal representations of latent volatility states, and how these representations emerge across the architecture. Using a controlled multivariate stochastic volatility process with cross-asset spillovers and varying volatility periods, we compare Transformers with MLP controls, vary training objectives and output heads, and analyze the learned computation using linear probes \citep{alain2016understanding}, stage-wise representation analysis, and targeted interventions \citep{meng2022locating, zhang2023activation, belrose2023leace}.

Our results reveal three main findings. First, across architectures and losses, both Transformers and MLPs exhibit a robust two-stage decomposition: hidden states encode substantial information about the next latent volatility state, while the output head maps that representation to squared return forecasts. Second, in Transformers, latent-state decodability emerges at identifiable architectural stages, and the stage at which it first appears depends on the volatility period. Long-cycle regimes become largely decodable immediately after $\ell^2$ normalization, whereas intermediate regimes require additional feed-forward refinement, and a short-period control regime depends strongly on positional embeddings. Third, in long-cycle regimes, the Transformer's latent-state computation simplifies into an explicit filter of the form in \Cref{eq:filter},
% \[
% \bm r_{1:T}\;\mapsto\;W\bm r_{1:T}\;\mapsto\;\frac{W\bm r_{1:T}}{\|W\bm r_{1:T}\|_2}\;\mapsto\;B\frac{W\bm r_{1:T}}{\|W\bm r_{1:T}\|_2}\mapsto\;\hat{\bm h}_{T+1},
% \]
showing that most latent-state information is extracted by a learned linear projection followed by nonlinear normalization. We further show that replacing the learned output head with a probe-aligned readout partially stabilizes performance under noisy MSE training, indicating that some forecast degradation arises from readout misalignment rather than failure to encode the latent state itself.

More broadly, we argue that stochastic volatility models offer a useful benchmark for mechanistic interpretability under noisy latent dynamics. Unlike many deterministic toy settings, they require implicit state inference under partial observability, but still permit controlled evaluation because the ground-truth latent state is known in simulation. This makes them a promising setting for studying how learned sequence models represent, transform, and use latent stochastic structure internally.

%=============================================================================
\section{Background and Setup}
\label{sec:background}

\paragraph{Multivariate stochastic volatility process.}
We use a controlled multivariate stochastic volatility (MSV) process \citep{harvey1994multivariate} with \(N=6\) assets, where returns $\bm r_t = [r_{1,t},\ldots,r_{N,t}]$ are noisy observations of a latent stochastic log-volatility state  \(\bm h_t=[h_{1,t},\ldots,h_{N,t}]^\top\):
\begin{equation}\label{eq:r_t}
    \bm r_t=\exp(\bm h_t/2)\odot \bm\varepsilon_t,
    \qquad 
    \bm\varepsilon_t\sim\mathcal N(\bm 0,\Sigma_\varepsilon),
\end{equation}
where \(\odot\) denotes elementwise multiplication. Thus,
\[
\mathbb E(\bm r_{t}^2 \mid \bm h_t)=\operatorname{Var}(\bm r_{t}\mid \bm h_t)=\exp(\bm h_{t}).
\] The latent log-volatility state follows
\[
    \bm h_{t+1}
    =
    \bm\mu + A(\bm h_t-\bm\mu)+\bm\eta_t,
    \qquad
    \bm\eta_t\sim\mathcal N(\bm 0,\Sigma_\eta),
\]
where the matrix $A$ governs cross-asset volatility spillovers. In the baseline calibration, $\bm{\mu} = -1$, $\Sigma_\varepsilon=I_N$ and $\Sigma_\eta=\sigma_\eta^2 I_N$ with
$\sigma_\eta=0.02$. To generate persistent cyclic spillovers between assets, we construct
\begin{equation}{\label{eq:A_matrix}}
    D=\operatorname{blkdiag}
\left(D_1(\omega),D_2(\omega),\ldots\right),
\qquad
D_j(\omega)=
\begin{pmatrix}
\cos(2\pi/\omega) & -\sin(2\pi/\omega)\\
\sin(2\pi/\omega) & \cos(2\pi/\omega)
\end{pmatrix},
\end{equation}
and set $A=QDQ^\top$ for an orthogonal matrix $Q$. We scale $A$ so that its largest absolute eigenvalue $\rho=0.999999$ in the main specifications. The period parameter
$\omega$ controls the cyclicity of the latent volatility dynamics and affects
cross-asset volatility correlations.

We split the simulated trajectory into $n=6000$ windows of length $L=22$, giving a sample size comparable to a realistic forecasting setting, using the first $T=21$ returns to predict the next-step squared returns:
\[
    y^{(k)}
    =
    \left(r_{1,T+1}^2,\ldots,r_{N,T+1}^2\right)^\top .
\]
Thus, each model  21 past returns for all assets and predicts the next
squared return for each asset.

\paragraph{Architectures.}
Our main model is a compact Transformer encoder with one layer and one attention head, chosen to facilitate interpretation. Each asset is treated as a token whose features are its length-$T$ return history. For asset $i$, the input $x_i\in\mathbb R^T$ is linearly embedded to $\mathbb{R}^{64}$, $\ell^2$-normalized, and combined with a learnable positional embedding:
\[
    \frac{x_i W_{\mathrm{in}}^\top}
    {\|x_i W_{\mathrm{in}}^\top\|_2}
    + p_i,
\]
The resulting token representations are passed through a single Transformer encoder layer with self-attention, a feed-forward network (FFN) with GELU activation and residual connections, and dropout. Let $z_i$ denote the final hidden representation for asset $i$. The output head maps $z_i$ to a positive variance forecast,
\[
    \widehat r_{i,T+1}^2
    =
    f\!\left(z_i W_{\mathrm{head}}^\top+b_{\mathrm{head}}\right),
\]
where $f$ is either the exponential or softplus function.

We compare the Transformer with an MLP baseline that processes each asset's return history independently, isolating the effect of Transformer-specific components such as attention-based cross-asset interaction.. The MLP uses three hidden layers of width 64 with GELU activations, together with the same output-head form as the Transformer. Full architectural details for both models are provided in \Cref{app:architecture}. Full architectural details for both the Transformer and MLP are provided in \Cref{app:architecture}.

% \paragraph{Loss functions.}
% We train models under two objectives. First, we use mean squared error (MSE) on
% next-step squared returns,
% % \[
% %     \mathcal L_{\mathrm{MSE}}
% %     =
% %     \frac{1}{N}\sum_{i=1}^N
% %     \left(\widehat r_{i,T+1}^2-r_{i,T+1}^2\right)^2 .
% % \]
% This objective treats squared returns as direct regression targets and does not
% use distributional knowledge of the MSV process.

% Second, when using the fact that returns are conditionally Gaussian with mean
% zero, we train using the Gaussian negative log-likelihood,
% \[
%     \mathcal L_{\mathrm{NLL}}
%     =
%     \frac{1}{N}\sum_{i=1}^{N}
%     \left[
%     \log(\widehat r_{i,T+1}^{2}+\delta)
%     +
%     \frac{r_{i,T+1}^{2}}{\widehat r_{i,T+1}^{2}+\delta}
%     \right],
% \]
% with $\delta=10^{-6}$ for numerical stability. Unlike MSE, this objective is
% directly aligned with conditional variance estimation: it is minimized when the
% predicted variance matches the true conditional variance. Comparing MSE and NLL
% therefore allows us to test whether latent-state recovery depends on the training
% objective or arises more generally from the architecture's need to solve the
% filtering problem implicit in MSV forecasting.

\paragraph{Loss functions.}
We train models with two objectives. The first is MSE on next-step squared returns, as \(\bm r_{T+1}^2\) is noisy proxy for conditional variance \(\mathbb E(\bm r_{t}^2 \mid \bm h_t)=\operatorname{Var}(\bm r_{t}\mid \bm h_t)\). The second is the Gaussian negative log-likelihood (NLL), which uses the fact that returns are mean-zero Gaussian:
\[
    \mathcal L_{\mathrm{NLL}}
    =
    \frac{1}{N}\sum_{i=1}^{N}
    \left[
    \log(\widehat r_{i,T+1}^{2}+\delta)
    +
    \frac{r_{i,T+1}^{2}}{\widehat r_{i,T+1}^{2}+\delta}
    \right],
\]
with \(\delta=10^{-6}\) for numerical stability. Unlike MSE, NLL \citep{nix1994estimating} is directly aligned with conditional variance estimation and is uniquely minimized when the predicted $\widehat{\bm r_{t}}^2$ matches the true conditional variance \(\mathbb E(\bm r_{t}^2 \mid \bm h_t)=\exp(\bm h_{t})\) shown in \Cref{app:nll_optimality}.
\section{Emergent Mechanisms in Latent Volatility Forecasting}
\label{sec:interpretation}
This section evaluates three hypotheses about the learned computation. 

\textbf{H1: Two-stage latent-state computation.}
The model first forms a hidden representation that encodes substantial information about the next latent log-volatility state \(\bm h_{T+1}\), and the output head then maps this representation to squared-return forecasts.

\textbf{H2: Stage-wise emergence in Transformers.}
In Transformers, latent-state decodability first appears at identifiable architectural stages, and the stage at which it emerges depends on the volatility period \(\omega\).

\textbf{H3: Explicit latent filter in long-cycle regimes.}
For sufficiently long volatility cycles, most latent-state information is extracted by a simple computation consisting of a learned linear projection followed by \(\ell^2\)-normalization, yielding an explicit latent-state filter.

\subsection{Robust Two-Stage Latent-State Decomposition}
\label{subsec:latent}

Across output activations and training objectives, both the Transformer and MLP exhibit a two-stage computation that mirrors the classical SV pipeline of latent-state inference followed by variance prediction. Because the observation equation is nonlinear, latent-state filtering in stochastic-volatility models is itself nontrivial and is typically carried out with approximate or simulation-based methods rather than by simple closed-form filtering \citep{jacquier1994bayesian,kim1998stochastic}.

\begin{enumerate}
    \item \textbf{Latent volatility inference:} The model first maps past returns to hidden representations $\bm{z}=[z_1,\ldots,z_N]\in\mathbb{R}^{N\times d_{\mathrm{model}}}$ that encode information about the next latent log-volatility state $\bm{h}_{T+1}$.
    
    \item \textbf{$\bm{r}_{T+1}^2$ Prediction.} The output head then maps this representation to forecast $\hat{\bm{r}}^2_{T+1}$.
\end{enumerate}

\paragraph{Linear probing.}
We fit a linear ridge-probe \citep{alain2016understanding} from each asset representation $z_i$ to the corresponding MSV model's latent state $h_{i,T+1}$. If the first stage performs latent-volatility inference, then $h_{i,T+1}$ should be linearly decodable from $z_i$. We also compare this probe with the latent-state fit induced by the learned linear head readout $z_i W_{\mathrm{head}}^{\top}$, which measures how well the output head uses the encoded latent information. Results for $\omega=365$ corresponding to year-long volatility cycle are shown in \Cref{tab:probe_head}.

\begin{table}[t]
\centering
\caption{Latent-state decodability and output-head alignment across activation functions, model architectures, and training objectives. Entries report median [IQR] over $n=10$ seeds.}
\label{tab:probe_head}
\small
\begin{tabular}{lllcc}
\toprule
Act. & Model & Loss & Ridge Probe $R^2$ & Learned Readout $R^2$ \\
\midrule
\multirow{4}{*}{exp}
& Transformer & NLL & 0.9069 \iqr{0.8827, 0.9175} & 0.8975 \iqr{0.8733, 0.9065} \\
& Transformer & MSE & 0.8794 \iqr{0.8699, 0.8929} & 0.7539 \iqr{0.7151, 0.7846} \\
& MLP         & NLL & 0.8985 \iqr{0.8800, 0.9098} & 0.8899 \iqr{0.8754, 0.9003} \\
& MLP         & MSE & 0.8034 \iqr{0.7789, 0.8332} & 0.7543 \iqr{0.7232, 0.7975} \\
\midrule
\multirow{4}{*}{spl}
& Transformer & NLL & 0.9007 \iqr{0.8844, 0.9122} & 0.8191 \iqr{0.7783, 0.8407} \\
& Transformer & MSE & 0.8810 \iqr{0.8684, 0.8963} & 0.8011 \iqr{0.6929, 0.8119} \\
& MLP         & NLL & 0.8863 \iqr{0.8660, 0.9053} & 0.7933 \iqr{0.7542, 0.8110} \\
& MLP         & MSE & 0.8211 \iqr{0.7811, 0.8436} & 0.7207 \iqr{0.6032, 0.7473} \\
\bottomrule
\end{tabular}
\end{table}

The probe $R^2$ remains high across architectures, losses, and activations, showing that $\bm{z}$ linearly encodes substantial information about $\bm{h}_{T+1}$. In contrast, the learned readout $R^2$ is considerably more sensitive to both the training objective and output activation. Under the softplus activation, alignment is consistently weaker, likely because it is mismatched with the exponential observation model in \Cref{eq:r_t}. Under MSE training, the learned output head often exhibits substantially lower $R^2$ than the probe despite strong latent-state decodability in $\bm{z}$. This suggests that part of the forecast degradation arises from readout misalignment rather than failure to encode the latent state itself.

\paragraph{Output-head intervention.}
We test whether the alignment gap is functionally relevant by replacing the learned output head with a ridge-probe readout trained on the final-layer representation $\bm{z}$. Under MSE training with exponential activation, the probe readout consistently improves MAE across volatility periods for both architectures, with larger gains for the Transformer in most regimes. MSE gains are more mixed, suggesting that the intervention primarily stabilizes typical forecast error rather than uniformly improving tail-sensitive squared error. The full alignment-gap and output-head intervention analysis are reported in \Cref{app:r2_gap}, and intervention results are summarized in \Cref{tab:repair}.

\begin{table}[H]
\centering
\caption{Output-head intervention under MSE training with exponential activation for selected $\omega$. Entries report median MSE and MAE over $n=10$ seeds for the learned head and a ridge readout from the final-layer representation. Positive gain indicates improvement. MAE improves consistently, while MSE effects are mixed.}
% the full table is given in \Cref{tab:full_repair}.}
\label{tab:repair}
\small
\begin{tabular}{llcccccc}
\toprule
 & & \multicolumn{3}{c}{MSE} & \multicolumn{3}{c}{MAE} \\
\cmidrule(lr){3-5}\cmidrule(lr){6-8}
Model & $\omega$ & Original & Probe & Gain & Original & Probe & Gain \\
\midrule
\multirow{3}{*}{Transformer}
 & 30  & 4.422 & 4.382 & $-1.39\%$ & 0.873 & 0.836 & $+5.28\%$ \\
 % & 90  & 3.138 & 3.204 & $-0.81\%$ & 0.758 & 0.742 & $+2.44\%$ \\
 & 178 & 7.470 & 7.469 & $-0.11\%$ & 0.987 & 0.958 & $+1.29\%$ \\
 & 365 & 5.158 & 5.128 & $+0.47\%$ & 0.856 & 0.853 & $+0.66\%$ \\
 % & 730 & 7.571 & 7.542 & $+0.52\%$ & 0.965 & 0.934 & $+1.51\%$ \\
\midrule
\multirow{3}{*}{MLP}
 & 30  & 4.384 & 4.397 & $-1.34\%$ & 0.856 & 0.834 & $+3.89\%$ \\
 % & 90  & 3.152 & 3.175 & $-0.20\%$ & 0.750 & 0.739 & $+3.30\%$ \\
 & 178 & 7.433 & 7.446 & $-0.87\%$ & 0.994 & 0.954 & $+3.05\%$ \\
 & 365 & 5.192 & 5.123 & $-0.01\%$ & 0.864 & 0.859 & $+1.18\%$ \\
 % & 730 & 7.535 & 7.714 & $-0.62\%$ & 0.939 & 0.949 & $+0.73\%$ \\
\bottomrule
\end{tabular}
\end{table}

% \paragraph{Signal-to-Noise Ratio.} The spectral radius $\rho$ of the state-transition matrix $A$ controls the signal-to-noise ratio in the latent process $\bm h_t$. When $\rho$ is small, $\bm h_t$ decays quickly toward zero and the observations are largely noise-dominated. As $\rho$ approaches $1$, the latent dynamics become more persistent and structured, making $\bm h_{T+1}$ easier to linearly decode from past observations, as shown in \Cref{app:signal_to_noise}.

\subsection{Stage-wise Emergence of Latent-State Decodability}{\label{subsec:latent-state emergence}}

In this section, we identify where inside the Transformer the latent state
$\bm{h}_{T+1}$ first becomes linearly decodable from the intermediate
representation.  We find that the earliest stage at which $\bm{h}_{T+1}$ is
linearly recoverable depends on the volatility period $\omega$ of
the AR component in \Cref{eq:A_matrix}.

\Cref{fig:stage_probe} reports the Ridge $R^2$ for predicting $h_{t+1}$ at
each component of the Transformer. For long-period dynamics ($\omega \in \{178,365\}$), linear decodability already emerges after $\ell^2$ normalization. For intermediate periods ($\omega \in \{30,90\}$), the activation in the FFN also produces substantial additional gains, suggesting that the FFN refines representations that attention alone leaves only partially structured. We also include a control regime with $\omega=21$, where the oscillation period matches the observation window length $T$ and one can memorize $h_{T+1}$ by periodicity. In this case, decodability remains near zero until positional embeddings are added, after which $R^2$ jumps sharply to approximately $1$. Attention contributes little to all periods other than $\omega = 30$.
\begin{figure}[t]
    \centering
    \begin{subfigure}[t]{0.54\linewidth}
        \centering
        \includegraphics[height=5cm]{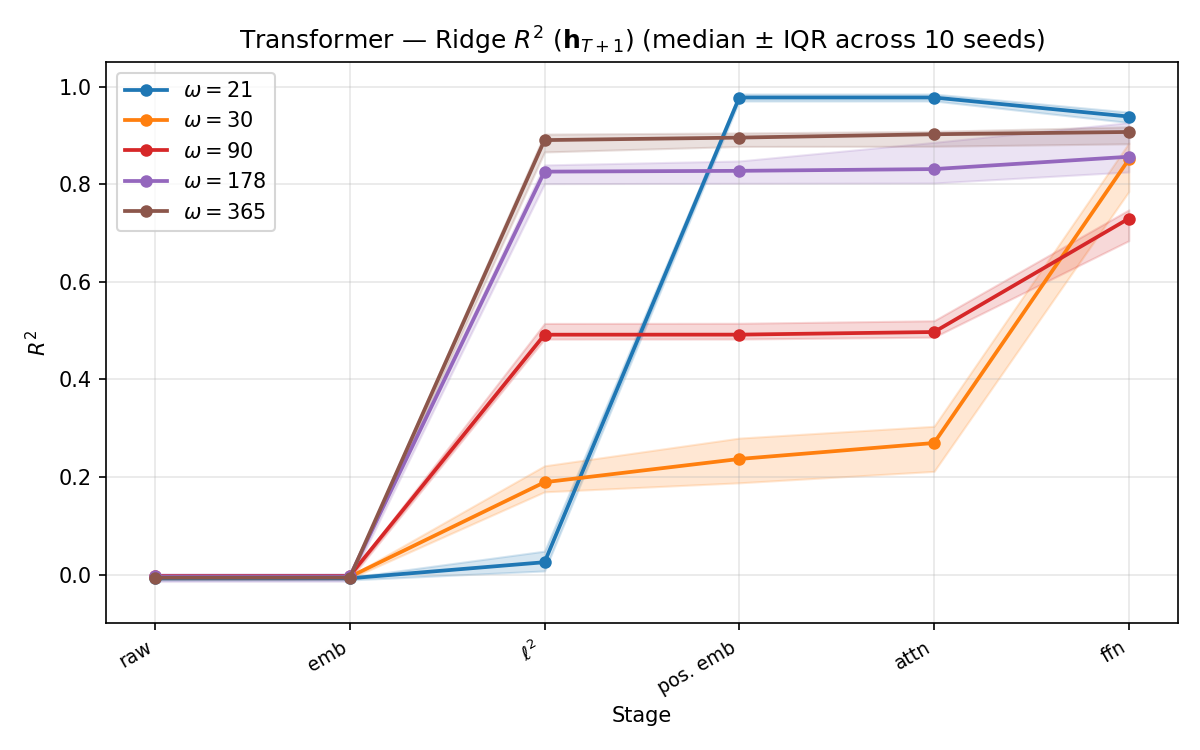}
        % \caption{Stage-wise ridge probe \(R^2\) for decoding \(h_{t+1}\) from intermediate Transformer representations across volatility periods \(\omega\).}
        \label{fig:stage_probe}
    \end{subfigure}
    \hfill
    \begin{subfigure}[t]{0.45\linewidth}
        \centering
        \includegraphics[height=5cm]{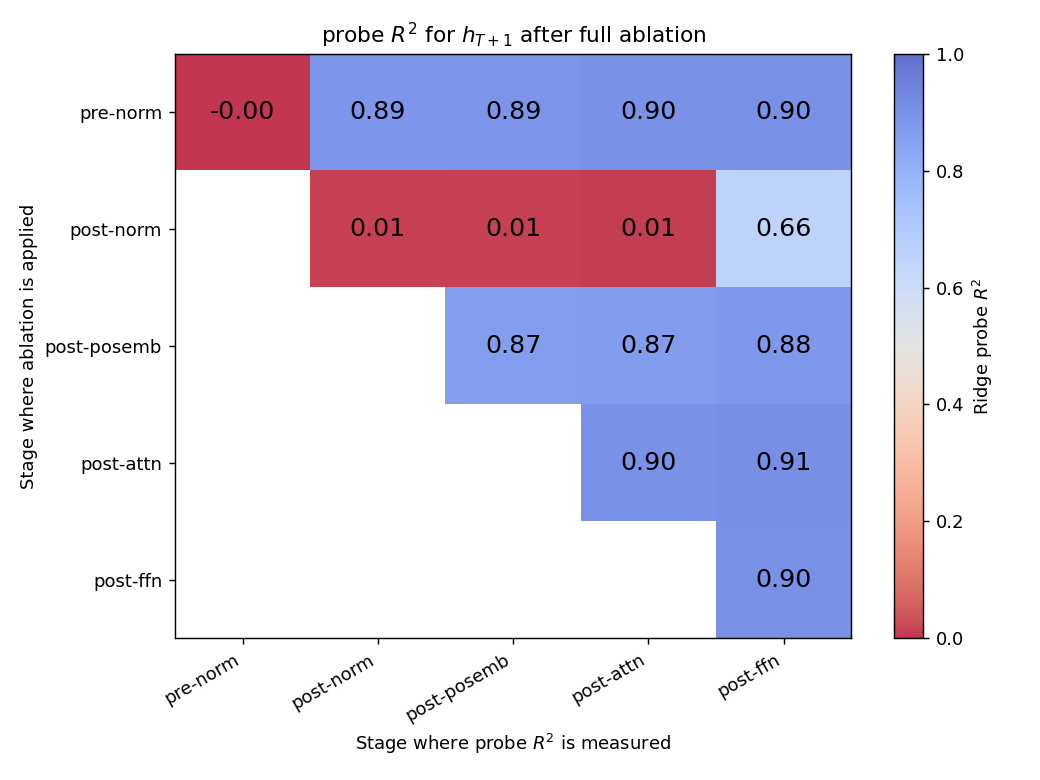}
        % \caption{Cross-stage effect of full aligned-direction ablation for \(\omega=365\). Rows indicate the ablation stage and columns the stage where ridge probe \(R^2\) for \(\bm h_{T+1}\) is measured after the perturbed forward pass.}
        \label{fig:cross_stage_ablation}
    \end{subfigure}
    \caption{Stage-wise emergence and causal relevance of latent-state decodability in the Transformer. Left: latent-state decodability emerges at different architectural stages depending on the volatility period \(\omega\). Right: for \(\omega=365\), ablating the post-\(\ell^2\)-norm representation produces the strongest downstream collapse in decodability, identifying it as the earliest causally relevant stage.}
    \label{fig:stage_probe_and_ablation}
\end{figure}

\paragraph{Less Structured Emergence in the MLP.} By contrast, the MLP exhibits a less modular stage-wise pattern: latent-state decodability rises after the first activation/linear block and then improves gradually across layers, rather than emerging at distinct architecture-specific stages. Even with an inserted $\ell^2$-normalization step, the MLP only partially mirrors the Transformer, with linear embedding and $\ell^2$ normalization accounting for most of the decodability in long-cycle regimes (\Cref{app:mlp_stage_probe}).
% As $\omega$ decreases, the dominant contribution also shifts from $\ell^2$-normalization alone to a shared contribution from $\ell^2$-normalization and the subsequent nonlinear activation.

\paragraph{Multi-Layer Transformer} In the two-layer Transformer, most latent-state computation still occurs early in the first layer in a similar manner, and the second layer contributes only modest refinement. Detailed layer-wise results are shown in \Cref{app:2_layer_transformer_stage_probe}.

\subsection{Explicit Latent Filters in Long-Cycle Regimes}
\label{subsec:explicit filter}

For long-cycle dynamics ($\omega \geq 365$), \Cref{fig:stage_probe} reveals a simple mechanism: Ridge $R^2$ for decoding $h_{t+1}$ jumps to approximately $0.9$ from $0$ after $\ell^2$ normalization, and changes little afterwards. Thus, most latent-state information is extracted by a single nonlinear normalization step applied after a learned linear projection:
\begin{equation}{\label{eq:filter}}
    \hat h_{i, t+1}
  =
  \underbrace{
  \frac{x_i W_{\mathrm{in}}^\top}
    {\|x_i W_{\mathrm{in}}^\top\|_2}
  }_{\ell^2\text{-normalized embedding}} 
  \underbrace{B}_{\text{ridge probe}},
\end{equation}
where $W_{\mathrm{in}}$ is the learned linear embedding and $B$ is the probe readout. We refer to this composition as an \emph{explicit latent-state filter}.

% \paragraph{Contrast with shorter cycles.}
% The explicit filter is specific to long-cycle regimes. For intermediate periods ($\omega=30,90$), $L^2$ normalization alone yields only weak decodability, and much of the improvement occurs after the feed-forward network. In these regimes, the latent phase changes appreciably within the observation window, so a static normalized embedding is insufficient; additional nonlinear processing is needed. 
% The control case $\omega=21$, where the cycle length matches the input window, behaves differently again: decodability remains near zero until positional embeddings are added, indicating that time-location information is essential in this regime.

\paragraph{Causal Perturbation of the Normalized Embedding} In \Cref{app:cross_stage_ablation}, we  show that the learned projection and the $\ell^2$ normalization step are jointly necessary.  For $\omega=365$, we ablate the component aligned with the ridge-probe direction at different stages and propagate the perturbed activation through the rest of the network. Ablating the post-$\ell^2$-norm representation causes a broad downstream collapse in decodability shown in \Cref{fig:cross_stage_ablation}, with only partial recovery at the FFN, whereas later ablations have much weaker effects. This identifies the $\ell^2$-normalized embedding as the earliest stage at which the latent-state signal becomes causally relevant for downstream computation. Further details are descried in \Cref{app:cross_stage_ablation}.

% \begin{figure}[H]
%     \centering
%     \includegraphics[width=0.45\linewidth]{pics/cross_stage_ablation.png}
%     \caption{Cross-stage effect of full aligned-direction ablation for $\omega=365$. Rows indicate the ablation stage and columns the stage where ridge probe $R^2$ for $\bm h_{T+1}$ is measured after the perturbed forward pass.}
%     \label{fig:cross_stage_ablation}
% \end{figure}

\bibliographystyle{plainnat}
\bibliography{ref}

\newpage
\appendix
\section{Architectural Details}
\label{app:architecture}

This section provides the full architectural details for the Transformer model and the MLP baseline used in the experiments.

\subsection{Transformer Architecture}
\label{app:transformer}

We model the conditional variance of asset returns using a Transformer encoder \citep{vaswani2017attention} that operates across assets. Given a panel of past returns
\[
x \in \mathbb{R}^{N \times T},
\]
where $x_i \in \mathbb{R}^T$ denotes the return history of asset $i$ over the previous $T$ time steps, each asset is treated as one token whose feature vector is its return trajectory.

The Transformer consists of three components: a linear embedding layer, a single Transformer encoder layer, and an output head. We use one attention head and one encoder layer throughout. Dropout with rate $0.1$ is applied in the model.

\paragraph{Linear embedding.}
For each asset $i$, the return history $x_i \in \mathbb{R}^{T}$ is first mapped into a $d_{\mathrm{model}}$-dimensional representation by a linear projection,
\begin{equation}
    \tilde{e}_i = x_i W_{\mathrm{in}}^\top,
\end{equation}
where $W_{\mathrm{in}} \in \mathbb{R}^{d_{\mathrm{model}} \times T}$. We then normalize the embedded representation by its $\ell^2$ norm,
\begin{equation}
    \tilde{e}_{i,\mathrm{norm}}
    =
    \frac{\tilde{e}_i}{\|\tilde{e}_i\|_2},
\end{equation}
and add a learnable asset-specific embedding,
\begin{equation}
    e_i = \tilde{e}_{i,\mathrm{norm}} + p_i,
\end{equation}
where $p_i \in \mathbb{R}^{d_{\mathrm{model}}}$ identifies asset $i$. In all experiments, we set $d_{\mathrm{model}}=64$.

\paragraph{Transformer encoder.}
Let
\[
H = [e_1,\ldots,e_N]
\]
denote the matrix of asset-token representations. The tokens are processed by a single Transformer encoder layer with one self-attention head, followed by a feed-forward network with GELU activation and residual connections. Specifically,
\begin{equation}
    z_i
    =
    \left[
    \mathrm{GELU}
    \left(
    \left(
    h_i + \mathrm{Attn}(H)
    \right) W_1^\top + b_1
    \right)
    \right] W_2^\top + b_2,
\end{equation}
where $\mathrm{Attn}(H)$ denotes the self-attention operator and GELU denotes the Gaussian error linear unit \citep{hendrycks2016gaussian}. We do not include layer normalization in the main specification, since it had little effect on forecasting performance in our experiments and removing it simplifies the interpretation of hidden representations.

\paragraph{Output head.}
A linear prediction head maps the final representation of each asset into a scalar variance forecast. For asset $i$,
\begin{equation}
\label{eq:output_linear_app}
    s_i = z_i W_{\mathrm{head}}^\top + b_{\mathrm{head}},
\end{equation}
and the predicted next-period squared return is
\begin{equation}
\label{eq:output_act_app}
    \widehat r_{i,T+1}^2 = f(s_i),
\end{equation}
where $W_{\mathrm{head}} \in \mathbb{R}^{1 \times d_{\mathrm{model}}}$ and $b_{\mathrm{head}}\in\mathbb{R}$. We consider both exponential and softplus choices for the output activation $f$ to ensure positive variance forecasts.

\subsection{MLP Baseline}
\label{app:mlp}

We use an MLP baseline to distinguish latent-state recovery through nonlinear univariate forecasting from recovery due to attention-based cross-asset aggregation. Unlike the Transformer, the MLP processes each asset independently and does not include asset-specific embeddings, self-attention, or cross-asset token interactions.

For each asset $i$, the input is its length-$T$ return history
$x_i \in \mathbb{R}^T$. The MLP consists of three hidden layers with GELU activations with width 64. Formally, the model computes
\begin{align}
    h_i^{(0)} &= x_i, \\
    a_i^{(\ell)}
    &=
    h_i^{(\ell-1)} W_\ell^\top + b_\ell,
    \qquad \ell = 1,2,3, \\
    h_i^{(\ell)}
    &=
    \mathrm{GELU}\left(a_i^{(\ell)}\right),
    \qquad \ell = 1,2,3, \\
    z_i &= h_i^{(3)} .
\end{align}

The output layer follows the same prediction-head structure as in the Transformer. Specifically,
\begin{equation}
    s_i = z_i W_{\mathrm{head}}^\top + b_{\mathrm{head}},
\end{equation}
and the predicted next-period squared return is
\begin{equation}
    \widehat r_{i,T+1}^2 = f(s_i),
\end{equation}
where $f$ is chosen to be either the exponential or softplus function to ensure positive squared return forecasts. The dropout configuration is kept the same as in the Transformer.

\section{Spillover matrix $A$}
\label{app:A_matrix}

We report the spillover matrix $A$ from \Cref{eq:A_matrix} for all volatility periods $\omega \in \{21,30,90,178,365,730\}$, denoted as $A_{\omega}$, considered in the paper.

\[
A_{21} =
\begin{pmatrix}
 0.9556 & -0.1398 &  0.0983 &  0.1905 & -0.1280 &  0.0710 \\
 0.1398 &  0.9556 & -0.2398 &  0.0852 & -0.0405 &  0.0304 \\
-0.0983 &  0.2398 &  0.9556 &  0.1167 & -0.0731 &  0.0276 \\
-0.1905 & -0.0852 & -0.1167 &  0.9556 & -0.0810 & -0.1523 \\
 0.1280 &  0.0405 &  0.0731 &  0.0810 &  0.9556 & -0.2387 \\
-0.0710 & -0.0304 & -0.0276 &  0.1523 &  0.2387 &  0.9556
\end{pmatrix}
\]

\[
A_{30} =
\begin{pmatrix}
 0.9781 & -0.0986 &  0.0693 &  0.1343 & -0.0903 &  0.0500 \\
 0.0986 &  0.9781 & -0.1692 &  0.0601 & -0.0286 &  0.0215 \\
-0.0693 &  0.1692 &  0.9781 &  0.0823 & -0.0515 &  0.0194 \\
-0.1343 & -0.0601 & -0.0823 &  0.9781 & -0.0571 & -0.1074 \\
 0.0903 &  0.0286 &  0.0515 &  0.0571 &  0.9781 & -0.1684 \\
-0.0500 & -0.0215 & -0.0194 &  0.1074 &  0.1684 &  0.9781
\end{pmatrix}
\]

\[
A_{90} =
\begin{pmatrix}
 0.9976 & -0.0331 &  0.0233 &  0.0451 & -0.0303 &  0.0168 \\
 0.0331 &  0.9976 & -0.0568 &  0.0202 & -0.0096 &  0.0072 \\
-0.0233 &  0.0568 &  0.9976 &  0.0276 & -0.0173 &  0.0065 \\
-0.0451 & -0.0202 & -0.0276 &  0.9976 & -0.0192 & -0.0360 \\
 0.0303 &  0.0096 &  0.0173 &  0.0192 &  0.9976 & -0.0565 \\
-0.0168 & -0.0072 & -0.0065 &  0.0360 &  0.0565 &  0.9976
\end{pmatrix}
\]

\[
A_{178} =
\begin{pmatrix}
 0.9994 & -0.0167 &  0.0118 &  0.0228 & -0.0153 &  0.0085 \\
 0.0167 &  0.9994 & -0.0287 &  0.0102 & -0.0048 &  0.0036 \\
-0.0118 &  0.0287 &  0.9994 &  0.0140 & -0.0087 &  0.0033 \\
-0.0228 & -0.0102 & -0.0140 &  0.9994 & -0.0097 & -0.0182 \\
 0.0153 &  0.0048 &  0.0087 &  0.0097 &  0.9994 & -0.0286 \\
-0.0085 & -0.0036 & -0.0033 &  0.0182 &  0.0286 &  0.9994
\end{pmatrix}
\]

\[
A_{365} =
\begin{pmatrix}
 0.9999 & -0.0082 &  0.0057 &  0.0111 & -0.0075 &  0.0041 \\
 0.0082 &  0.9999 & -0.0140 &  0.0050 & -0.0024 &  0.0018 \\
-0.0057 &  0.0140 &  0.9999 &  0.0068 & -0.0043 &  0.0016 \\
-0.0111 & -0.0050 & -0.0068 &  0.9999 & -0.0047 & -0.0089 \\
 0.0075 &  0.0024 &  0.0043 &  0.0047 &  0.9999 & -0.0139 \\
-0.0041 & -0.0018 & -0.0016 &  0.0089 &  0.0139 &  0.9999
\end{pmatrix}
\]

\[
A_{730} =
\begin{pmatrix}
 1.0000 & -0.0041 &  0.0029 &  0.0056 & -0.0037 &  0.0021 \\
 0.0041 &  1.0000 & -0.0070 &  0.0025 & -0.0012 &  0.0009 \\
-0.0029 &  0.0070 &  1.0000 &  0.0034 & -0.0021 &  0.0008 \\
-0.0056 & -0.0025 & -0.0034 &  1.0000 & -0.0024 & -0.0044 \\
 0.0037 &  0.0012 &  0.0021 &  0.0024 &  1.0000 & -0.0070 \\
-0.0021 & -0.0009 & -0.0008 &  0.0044 &  0.0070 &  1.0000
\end{pmatrix}
\]

\section{Optimality of Gaussian NLL for Conditional Variance}
\label{app:nll_optimality}

\begin{theorem}
Fix an asset \(i\) and condition on the latent state \(\bm h_{T+1}\). Under the observation model in \Cref{eq:r_t},
\[
r_{i,T+1}\mid \bm h_{T+1}\sim \mathcal N\!\left(0,\exp(h_{i,T+1})\right).
\]
Consider the per-asset Gaussian negative log-likelihood
\[
\ell\!\left(\widehat r_{i,T+1}^2; r_{i,T+1}\right)
=
\log \widehat r_{i,T+1}^2
+
\frac{r_{i,T+1}^2}{\widehat r_{i,T+1}^2},
\qquad \widehat r_{i,T+1}^2>0.
\]
Then the conditional expected loss
\[
\mathbb E\!\left[
\ell\!\left(\widehat r_{i,T+1}^2; r_{i,T+1}\right)
\mid \bm h_{T+1}
\right]
\]
is uniquely minimized at
\[
\widehat r_{i,T+1}^2=\exp(h_{i,T+1}).
\]
\end{theorem}

\begin{proof}
Conditioning on \(\bm h_{T+1}\), we have
\[
\mathbb E\!\left[r_{i,T+1}^2 \mid \bm h_{T+1}\right]
=
\operatorname{Var}(r_{i,T+1}\mid \bm h_{T+1})
=
\exp(h_{i,T+1}).
\]
Therefore,
\[
\mathbb E\!\left[
\ell\!\left(\widehat r_{i,T+1}^2; r_{i,T+1}\right)
\mid \bm h_{T+1}
\right]
=
\log \widehat r_{i,T+1}^2
+
\frac{\exp(h_{i,T+1})}{\widehat r_{i,T+1}^2}.
\]
Define
\[
J(v)=\log v+\frac{\exp(h_{i,T+1})}{v},
\qquad v>0.
\]
Then
\[
J'(v)=\frac{1}{v}-\frac{\exp(h_{i,T+1})}{v^2}
=\frac{v-\exp(h_{i,T+1})}{v^2}.
\]
Thus the unique critical point is
\[
v=\exp(h_{i,T+1}).
\]
Moreover,
\[
J''(v)
=
-\frac{1}{v^2}
+
\frac{2\exp(h_{i,T+1})}{v^3},
\]
so at \(v=\exp(h_{i,T+1})\),
\[
J''\!\left(\exp(h_{i,T+1})\right)
=
\frac{1}{\exp(2h_{i,T+1})}>0.
\]
Hence \(
\widehat r_{i,T+1}^2=\exp(h_{i,T+1})
\)
uniquely minimizes the conditional expected Gaussian NLL.
\end{proof}

\section{Consistent Probe--Readout $R^2$ Gap Under MSE Loss}
\label{app:r2_gap}

In \Cref{tab:probe_head}, we reported a substantial gap between the hidden-state probe $R^2$ and the linear head $R^2$ at $\omega=365$. \Cref{fig:r2_gap_sweep} shows that this gap persists across all volatility periods $\omega \in \{30,90,178,365,730\}$ for both the Transformer and the MLP under MSE training. For both architectures, the hidden probe $R^2$ is consistently higher than the linear head $R^2$. The gap narrows as $\omega$ increases, and both probe $R^2$ and head $R^2$ improve with longer volatility cycles. The Transformer exhibits a consistently smaller alignment gap than the MLP, especially at $\omega=30$ and $\omega=90$, suggesting that under MSE training its learned head remains better aligned with the latent-state direction encoded in the hidden representation.

\begin{figure}[H]
    \centering
    \includegraphics[width=\linewidth]{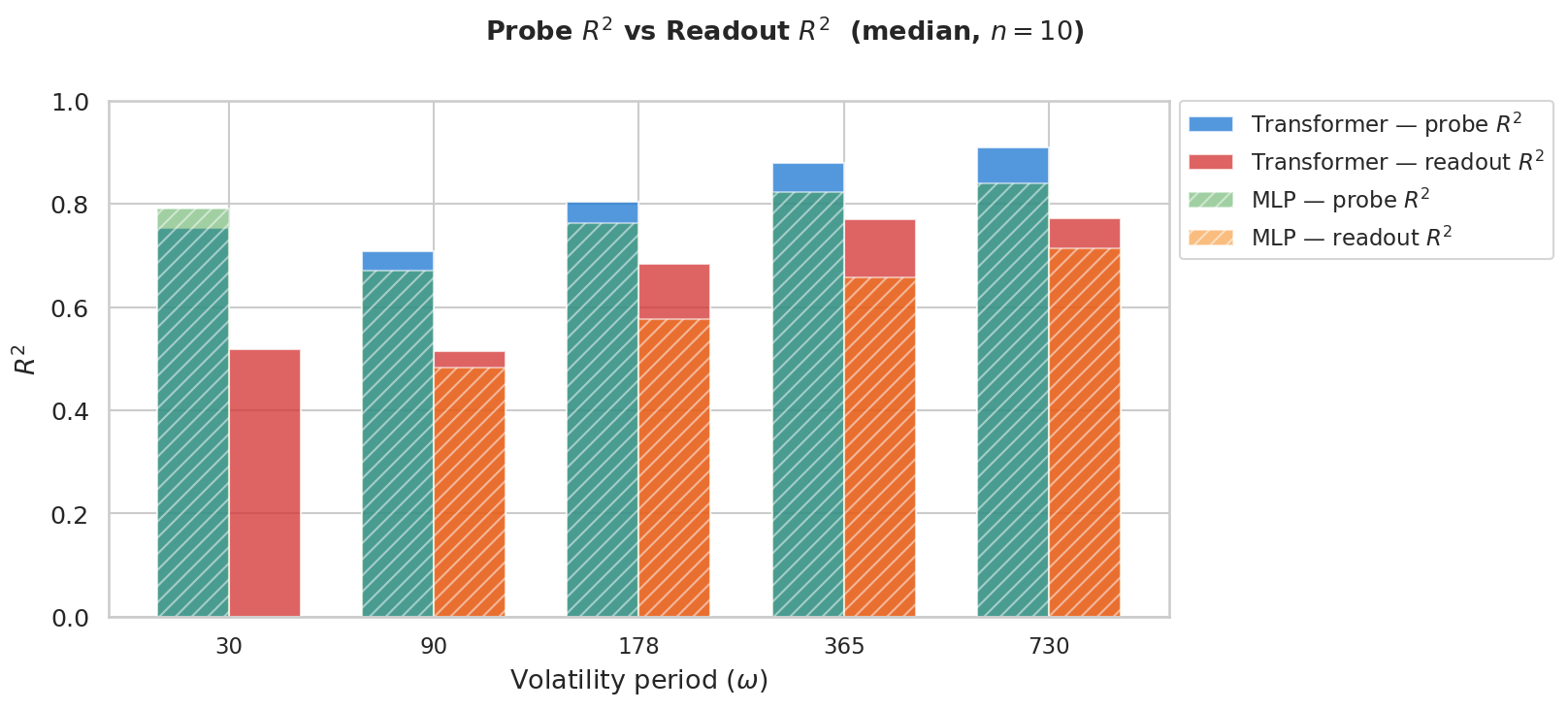}
    \caption{Probe $R^2$ and linear head $R^2$ for decoding $\bm h_{T+1}$ from the final representation under MSE training, shown across volatility periods for the Transformer and MLP (median over $n=10$ seeds). In both architectures, probe $R^2$ remains above head $R^2$ across all periods, with smaller in long-cycle regimes. The $R^2$ in Transformer are consistently higher.}
    \label{fig:r2_gap_sweep}
\end{figure}

% We also report the full version of \Cref{tab:repair} in \Cref{tab:full_repair}, covering a wider range of $\omega$. Across volatility periods, replacing the learned linear output head with a ridge readout consistently improves MAE, while the effect on MSE remains mixed.
% \begin{table}[H]
% \centering
% \label{tab:full_repair}
% \small
% \begin{tabular}{llcccccc}
% \toprule
%  & & \multicolumn{3}{c}{MSE} & \multicolumn{3}{c}{MAE} \\
% \cmidrule(lr){3-5}\cmidrule(lr){6-8}
% Model & $\omega$ & Original & Probe & Gain & Original & Probe & Gain \\
% \midrule
% \multirow{5}{*}{Transformer}
%  & 30  & 4.074 & 4.311 & $-3.19\%$ & 0.848 & 0.812 & $+4.44\%$ \\
%  & 90  & 3.651 & 3.731 & $-3.11\%$ & 0.835 & 0.787 & $+3.55\%$ \\
%  & 178 & 3.506 & 3.514 & $+0.49\%$ & 0.775 & 0.746 & $+3.40\%$ \\
%  & 365 & 3.487 & 3.447 & $+0.60\%$ & 0.781 & 0.758 & $+2.01\%$ \\
%  & 730 & 3.463 & 3.453 & $+0.42\%$ & 0.792 & 0.765 & $+1.61\%$ \\
% \midrule
% \multirow{5}{*}{MLP}
%  & 30  & 4.096 & 4.212 & $-2.69\%$ & 0.837 & 0.803 & $+3.00\%$ \\
%  & 90  & 3.658 & 3.757 & $-1.01\%$ & 0.857 & 0.808 & $+3.92\%$ \\
%  & 178 & 3.498 & 3.536 & $-1.66\%$ & 0.807 & 0.778 & $+2.53\%$ \\
%  & 365 & 3.398 & 3.536 & $-3.18\%$ & 0.803 & 0.794 & $+1.87\%$ \\
%  & 730 & 3.426 & 3.519 & $-2.24\%$ & 0.810 & 0.799 & $+0.50\%$ \\
% \bottomrule
% \end{tabular}
% \end{table}

\section{Less Structured Emergence in the MLP}
\label{app:mlp_stage_probe}

In the Transformer, the stage at which the latent state $\bm{h}_{T+1}$ first becomes linearly decodable depends strongly on the volatility cycle as shown in \Cref{fig:stage_probe}. This stage-specific pattern is much less pronounced in a standard MLP. As shown in \Cref{fig:mlp_stage_probe}, latent-state decodability in the MLP typically arises in a more diffuse and less modular manner than the Transformer. It arises after the first nonlinear block and then improves more gradually across subsequent layers, rather than emerging sharply at architecture-specific components. 

\begin{figure}[H]
    \centering
    \includegraphics[width=0.75\linewidth]{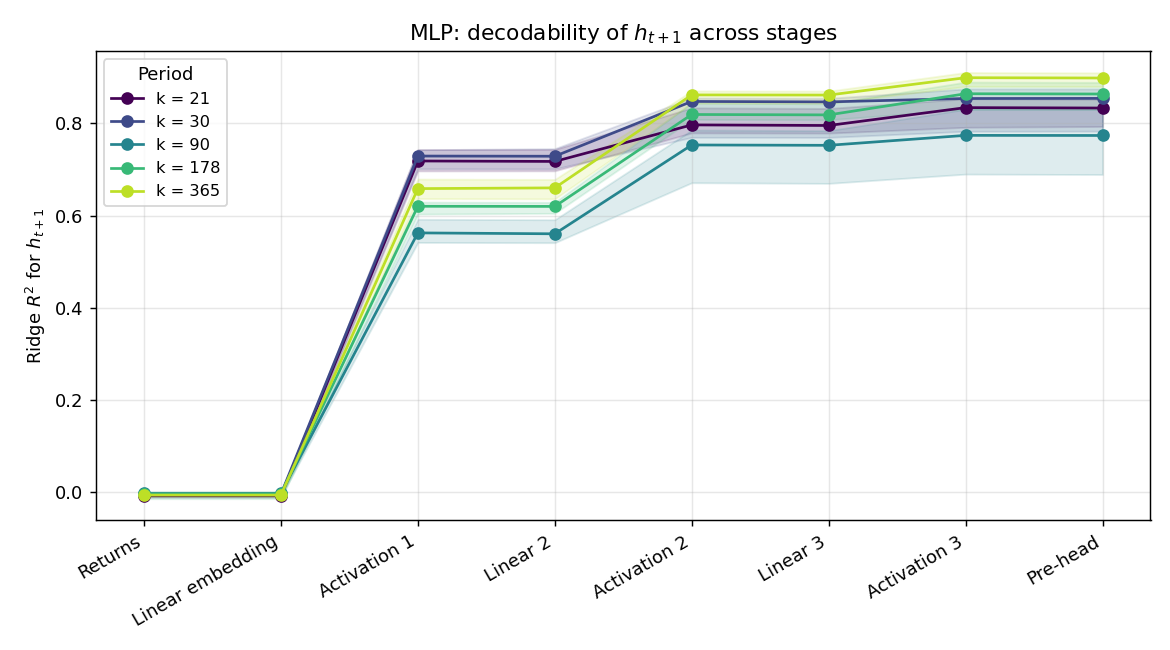}
    \caption{Stage-wise ridge probe $R^2$ for decoding $\bm{h}_{T+1}$ from intermediate MLP representations across volatility periods. Unlike the Transformer, the MLP exhibits a smoother and less localized rise in decodability, with weaker dependence on volatility period.}
    \label{fig:mlp_stage_probe}
\end{figure}

We also consider an MLP control with an $\ell^2$-normalization layer inserted after the initial linear embedding, mirroring the corresponding preprocessing step in the Transformer. As shown in \Cref{fig:mlp_normed_stage_probe}, this normalized MLP exhibits a similar pattern: in long-cycle regimes, the linear embedding together with $\ell^2$-normalization accounts for most of the decodability of $h_{t+1}$. As $\omega$ decreases, the dominant contribution shifts from $\ell^2$-normalization alone to a shared contribution from $\ell^2$-normalization and the subsequent nonlinear activation.

\begin{figure}[H]
    \centering
    \includegraphics[width=0.75\linewidth]{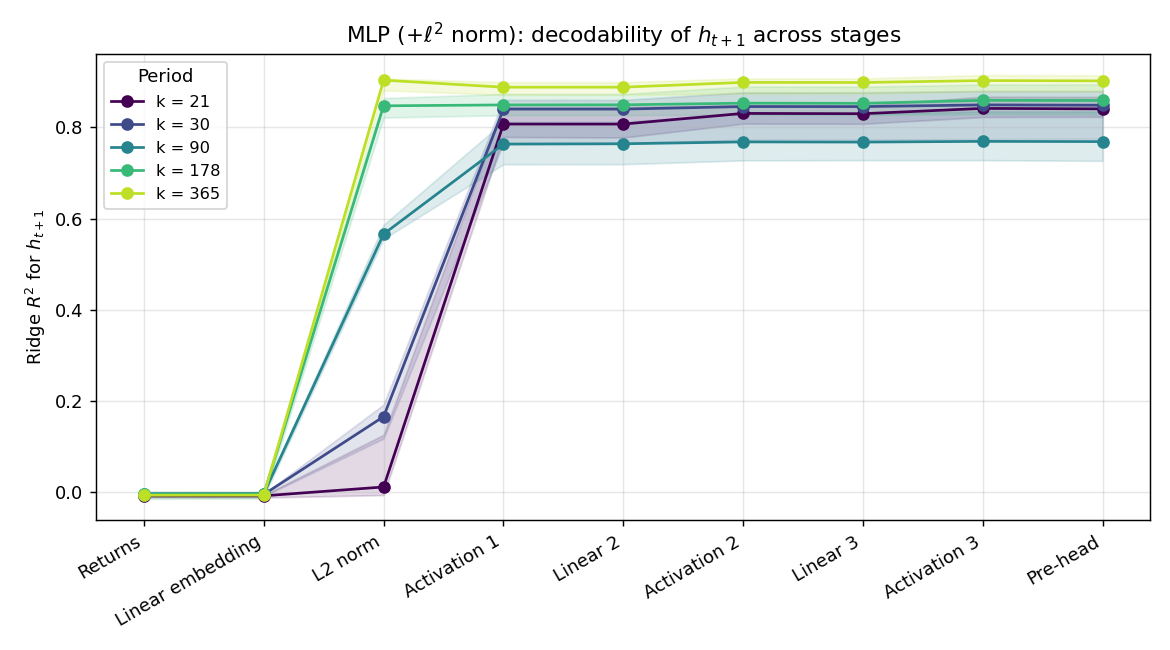}
    \caption{Stage-wise ridge probe $R^2$ for decoding $\bm{h}_{T+1}$ in an MLP with $\ell^2$ normalization after the initial linear layer.}
    \label{fig:mlp_normed_stage_probe}
\end{figure}

\section{Stage-Wise Results for Two-Layer Transformers}
\label{app:2_layer_transformer_stage_probe}

We also examine whether a deeper Transformer exhibits a different stage-wise emergence pattern. \Cref{fig:2_layer_transformer_stage_probe} shows that the main decodability transitions still occur in the first layer. In our setting, increasing depth from one layer to two does not substantially change where the latent-state representation is formed.

\begin{figure}[H]
    \centering
    \includegraphics[width=0.75\linewidth]{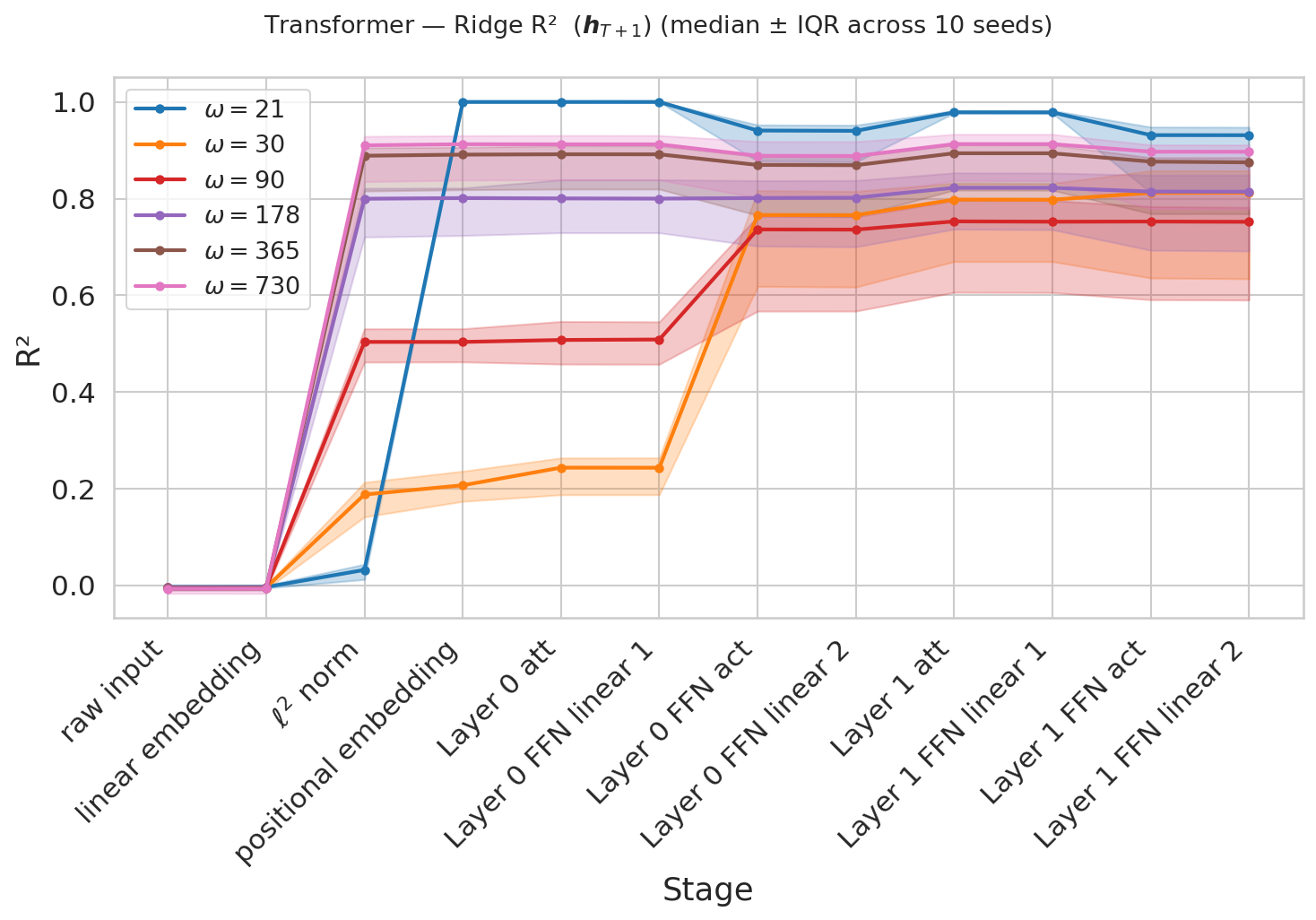}
    \caption{Stage-wise ridge probe $R^2$ for decoding $h_{T+1}$ from intermediate representations in a two-layer Transformer across volatility periods. The main gains in decodability occur before or within the first layer, while the second layer yields only limited additional improvement.}
    \label{fig:2_layer_transformer_stage_probe}
\end{figure}

\section{Impact of Signal-to-Noise Ratio}{\label{app:signal_to_noise}}
The spectral radius $\rho$ of the state-transition matrix $A$ controls the persistence of the latent process $\bm h_t$, and hence its signal-to-noise ratio. Smaller $\rho$ leads to faster decay and weaker recoverable structure, while $\rho$ near $1$ yields more persistent dynamics and higher latent-state decodability. Consistent with this, ridge probe $R^2$ rises sharply as $\rho$ increases and then largely plateaus for the largest values; representative trajectories are shown in \Cref{fig:time_series}, with quantitative results in \Cref{tab:signal_to_noise}.
\begin{table}[H]
\centering
\caption{Mean ridge probe $R^2$ for decoding $\textbf{h}_{T+1}$ from $\bm z$ across $\rho$ for the Transformer and MLP ($n=5$ seeds).}
\label{tab:signal_to_noise}
\begin{tabular}{lcc}
\toprule
$\rho$ & Transformer Mean $R^2$ & MLP Mean $R^2$ \\
\midrule
$0.9999$    & $0.286$ & $0.289$ \\
$0.99999$   & $0.873$ & $0.872$ \\
$0.999999$  & $0.901$ & $0.897$ \\
$0.9999999$ & $0.903$ & $0.899$ \\
\bottomrule
\end{tabular}
\end{table}

\begin{figure}[H]
    \centering
    \includegraphics[width=\linewidth]{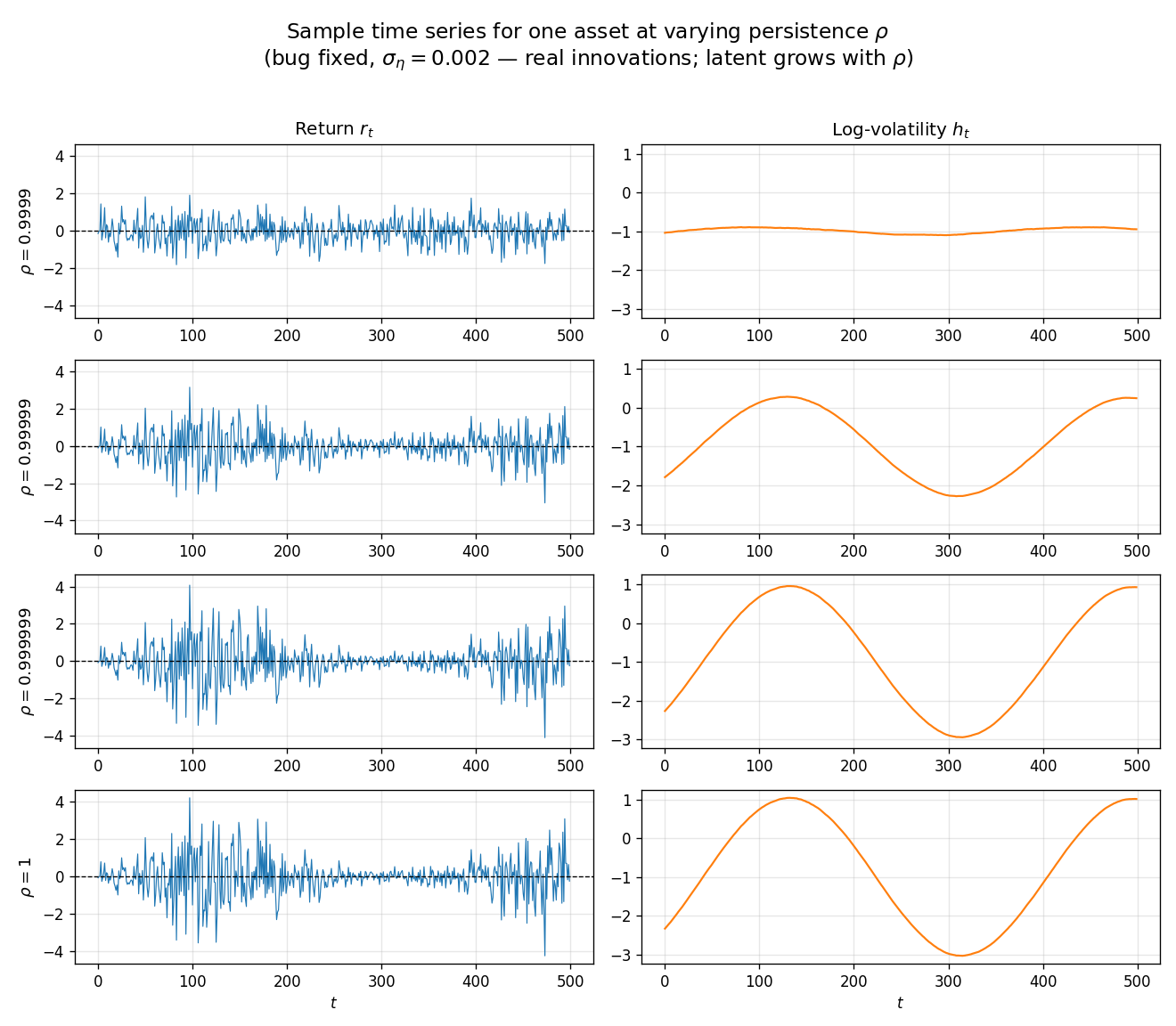}
    \caption{Representative latent trajectories for different spectral radii $\rho$. Smaller $\rho$ leads to faster decay and weaker temporal structure, while $\rho$ close to $1$ produces more persistent latent dynamics.}
    \label{fig:time_series}
\end{figure}

\section{Details on Explicit Latent Filters in Long-Cycle Regimes}
\label{app:cross_stage_ablation}

\paragraph{Both components are essential}
To isolate which components of the input-embedding filter block are load-bearing, we compare three conditions: 
\begin{enumerate}
    \item the full filter, which applies the learned linear projection $W_{\mathrm{in}}$ followed by $\ell^2$ normalization; 
    \item an $\ell^2$-only baseline, which normalizes the raw return window directly and bypasses the learned projection; 
    \item and a linear-only baseline, which uses the projected representation without normalization.
\end{enumerate}
As reported in \Cref{tab:filter_ablation}, only the full composition yields high decodability of $\bm{h}_{t+1}$, while both ablated variants remain near zero. This shows that the learned projection and the normalization step are jointly necessary: neither one alone is sufficient to make the latent state linearly decodable.

\begin{table}[H]
  \centering
  \caption{Ridge $R^2$ for probing $h_{t+1}$ from transformer representations under three ablation conditions (median [IQR], $n=10$ seeds).}
  \label{tab:filter_ablation}
  \small
  \begin{tabular}{lcc}
    \toprule
    Condition & $\omega=365$ & $\omega=730$ \\
    \midrule
    Linear $\to$ $\ell^2$ norm $\to$ Linear
    & $0.890$ \iqr{0.866,\ 0.903}
    & $0.921$ \iqr{0.914,\ 0.933} \\

    $\ell^2$ norm on raw returns $\to$ Linear
    & $-0.004$ \iqr{-0.005,\ -0.002}
    & $-0.003$ \iqr{-0.004,\ -0.002} \\

    Linear embedding only $\to$ Linear
    & $-0.006$ \iqr{-0.009,\ -0.003}
    & $-0.005$ \iqr{-0.015,\ -0.000} \\
    \bottomrule
  \end{tabular}
\end{table}

\paragraph{Causal perturbation of the post-norm representation.}
We next test whether the post-\(\ell^2\)-norm embedding is causally upstream of the forecast. Following causal intervention methods from mechanistic interpretability, in the spirit of activation patching, causal tracing and concept-erasure \citep{meng2022locating,heimersheim2024use,belrose2023leace}, we remove a varying fraction \(\alpha\) of the component aligned with the latent-state direction at a chosen stage and then propagate the perturbed activation through the remainder of the network unchanged. We apply this intervention to the post-input-embedding, post-\(\ell^2\)-norm, post-positional-embedding, post-attention, and post-FFN representations, and measure both the downstream change in latent-state decodability (\Cref{fig:cross_stage_ablation}) and the change in forecast error (\Cref{fig:var365_forecast_delta}). For \(\omega=365\), perturbing the post-\(\ell^2\)-norm representation produces the largest downstream collapse in decodability and the strongest increase in forecast MSE and MAE, whereas perturbations at later stages have substantially weaker effects. This identifies the \(\ell^2\)-normalized embedding as the earliest stage at which the latent-state signal becomes causally relevant for the model's prediction.

\begin{figure}[H]
  \centering
  \includegraphics[width=\linewidth]{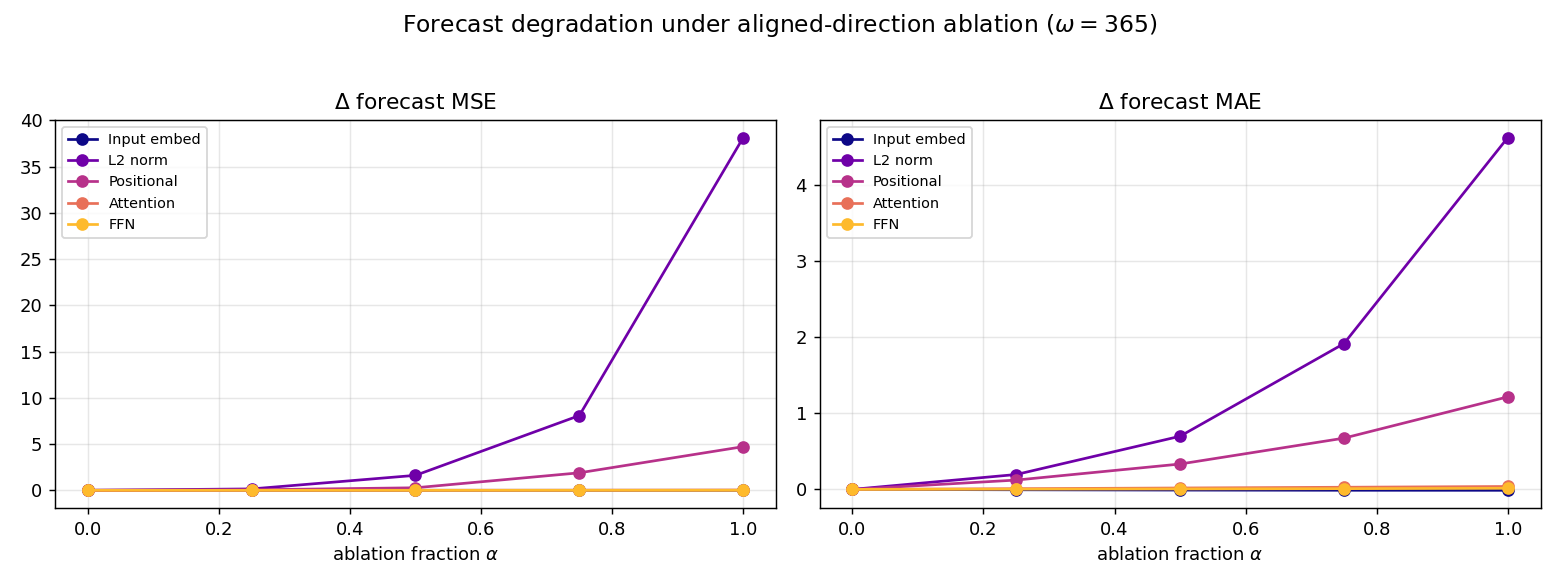}
  \caption{Forecast error change ($\Delta\mathrm{MSE}$ and $\Delta\mathrm{MAE}$,
           ablated minus baseline) as a function of ablation strength $\alpha$
           at $\omega{=}365$, one curve per intervention stage (mean $\pm$ std,
           $n{=}10$ seeds). Intervening at after $\ell^2$-norm causes the 
           largest degradation; later stages other than positional embedding have negligible effect.}
  \label{fig:var365_forecast_delta}
\end{figure}

\section{Related Work}
\label{app:related_work}

% \paragraph{Mechanistic interpretability and transformer circuits.}
% Mechanistic interpretability aims to reverse engineer the internal computations implemented by neural networks, often by identifying circuits, features, or causal pathways inside trained models \citep{olah2020zoom,elhage2021mathematical,olsson2022induction,wang2022interpretability,bricken2023monosemanticity,nanda2023progress,chughtai2023toy}. 
% Much of this work studies language models or controlled deterministic tasks, where the relevant computation is often symbolic or directly observable. 
% Our setting differs by studying noisy, partially observed stochastic dynamics, where the target computational object is a latent volatility state that is hidden from the model but known to the researcher through simulation.

\paragraph{Probing and causal interventions.}
Our stage-wise analysis builds on linear probing methods, which test what information is decodable from intermediate representations \citep{alain2016understanding,hewitt2019designing,belinkov2022probing}. 
Because probe decodability alone does not establish that the model uses the decoded information, we complement probes with activation-level interventions, in the spirit of causal mediation, activation patching, causal scrubbing, and causal-abstraction approaches \citep{vig2020investigating,chan2022causal,geiger2024finding,geiger2025causal,zhang2023activation}. 
In contrast to many concept-level analyses, our probe target is the known latent state of the data-generating process.

\paragraph{Stochastic volatility and latent-state filtering.}
Stochastic-volatility models represent volatility as an unobserved stochastic state that must be inferred from noisy returns \citep{taylor1986modelling,shephard1996statistical,harvey1994multivariate,kim1998stochastic,jacquier2002bayesian}. 
Since returns depend nonlinearly on latent log-volatility, filtering and likelihood evaluation generally require approximation, simulation, or Bayesian methods, rather than closed-form Kalman filtering \citep{kitagawa1996monte,durbin2012time}. 
We use this classical latent-state setting not to introduce a new estimator, but as a controlled benchmark for testing whether neural forecasters internally recover information about the latent volatility state.

\paragraph{Transformers for time series and financial forecasting.}
A growing literature applies Transformer architectures to time-series forecasting, motivated by their ability to capture long-range temporal dependence and interactions across variables through self-attention \cite{wen2022transformers,zhao2026survey,lim2021temporal,zhou2021informer}. 
In finance, Transformers have been used for return prediction, market-risk forecasting, and volatility estimation \cite{yang2020html,zeng2023financial,li2024master}. 
These applications are natural in multivariate financial settings, where asset returns and volatilities interact through common factors, market-wide shocks, and spillover effects \cite{fama1993common,andersen2003modeling,barigozzi2019identification}. 
Most existing work, however, emphasizes predictive performance. 
Our focus is instead mechanistic: we ask whether a Transformer trained for volatility forecasting internally recovers the latent volatility state and where this computation emerges in the architecture.

\paragraph{Domain Knowledge Discovery Via Mechanistic Interpretability} 
A growing line of work studies mechanistic interpretability not only as a tool for explaining model behavior, but also as a route to domain knowledge discovery. In enumerative geometry, \citet{hashemi2025can} train a Transformer to compute $\psi$-class intersection numbers and report interpretability evidence that the model implicitly captures structural objects. In combinatorics, \citet{huang2025black} propose a workflow in which a Transformer trained on paired Dyck paths is interpreted to derive a new algorithmic description of the well known zeta map. Our setting is aligned with this perspective, but focuses on latent-state recovery in stochastic dynamical systems rather than symbolic or combinatorial structure.

\section{Limitations And Broader Implications}

\paragraph{Limitations.}
First, all experiments are conducted in synthetic MSV models with known latent states. This makes controlled mechanistic analysis possible, but leaves open whether the same representational structure and stage-wise emergence persist in real financial time series, which may include heavier tails, regime shifts, and model misspecification. Second, our evidence is still based primarily on linear probes, stage-wise readouts, and targeted aligned-direction interventions. These results establish latent-state recoverability and partial mechanism identification, but not a complete circuit-level account of how the computation is implemented. Third, the architectures studied here are intentionally small. While this supports interpretability, it remains unclear how well the same mechanisms scale to larger forecasting models. Finally, because the MLP baselines also recover substantial latent information in some regimes, our strongest Transformer-specific claim is not latent-state recovery itself, but rather the more localized stage-wise emergence and the explicit long-cycle filter structure.

\paragraph{Broader Implications.}
More broadly, our results suggest that financial time series may provide a useful benchmark for mechanistic interpretability under noisy observations and latent stochastic dynamics. Unlike many synthetic mechanistic tasks, stochastic-volatility models contain irreducible uncertainty and partial observability, forcing models to perform implicit filtering rather than exact algorithmic computation. Understanding how neural sequence models internally represent and manipulate latent-state information in such settings may provide a useful bridge between mechanistic interpretability, time-series modeling, and state-space inference.
%=============================================================================

\end{document}